\documentclass[twoside]{article}

\usepackage{etoolbox}
\newtoggle{long-version}
\toggletrue{long-version}
\usepackage{aistats2020}
\usepackage[round]{natbib}

\usepackage{amsthm}
\usepackage{verbatim}
\usepackage{amsmath, amssymb}
\usepackage{mathtools}
\usepackage{algorithm}
\usepackage{algpseudocode}
\usepackage{caption}
\usepackage{subcaption}
\usepackage{multicol}
\usepackage[shortlabels]{enumitem}
\setitemize{wide}

\newcommand{\mdp}{m}
\newcommand{\hmdp}{\bar{m}}
\newcommand{\smdp}{\mathcal{M}}
\newcommand{\sS}{\mathcal{S}}
\newcommand{\sA}{\mathcal{A}}
\newcommand{\reset}{\mathrm{RESET}}
\newtheorem{definition}{Definition}
\newtheorem{assumption}{Assumption}
\newcommand{\sSa}{\mathcal{S}^\rightarrow}
\newtheorem{theorem}{Theorem}
\newtheorem{lemma}{Lemma}
\newcommand{\sK}{\mathcal{K}}
\newcommand{\sP}{\mathcal{P}}
\newcommand{\ucbexp}{\textsc{UcbExplore}}

\newcommand{\de}{b}
\newcommand{\vic}{v}
\newcommand{\bP}{\mathbb{P}}

\newcommand{\checkterm}{\alpha}
\newcommand{\sStreams}{\mathrm{STM}}
\newcommand{\stream}{\mathrm{p}}
\newcommand{\defined}{\coloneqq}
\newcommand{\sNExp}{\mathcal{T}}

\newcommand{\alg}{\textsc{MNM}}
\begin{document}



\twocolumn[

\aistatstitle{Autonomous exploration for navigating in non-stationary CMPs}

\aistatsauthor{ Pratik Gajane$^1$ \And Ronald Ortner$^1$ \And  Peter Auer$^1$ \And Csaba Szepesv{\'a}ri$^{2,3}$ }

\aistatsaddress{ $^1$Montanuniversit{\"a}t Leoben \And  $^2$DeepMind \And $^3$University of Alberta } ]

\begin{abstract}
 We consider a setting in which the objective is to learn to navigate in a controlled Markov  process (CMP) where
% extrinsic rewards are not present and 
 transition probabilities may abruptly change. 
 For this setting, we propose a performance measure called \textit{exploration steps} which counts the time steps at which the learner lacks \textit{sufficient} knowledge to navigate its environment \textit{efficiently}. We devise a learning meta-algorithm, \alg, and prove an upper bound on the exploration steps in terms of the number of changes.  
\end{abstract}

\section{Introduction}
The ability to quickly learn to reliably control one's environment is core to the functionality of intelligent agents.
Throughout the last decades, much work has been devoted to the design and testing of various algorithms targeted at this task, under various names such as 
learning using intrinsic motivation, intrinsic reward, curiosity-driven learning, etc.
A necessarily incomplete sample of prior works in the area includes that of \citet{Schmidhuber1991,Singh2004,Oudeyer2007,Oudeyer2007b,Baranes2009,Schmidhuber2010,Singh2010,lopes2012exploration,gottlieb2013information,StLeAb15,houthooft2016variational,AchSa17,ostrovski2017count,pathak2017curiosity,haber2018learning,BuEdPa19,AzPiPi19,HaKaSiVS19}. 
Conceptually, the problem can be thought of as learning to reliably navigate an unknown environment. In this article we focus on this problem, and in particular, on learning to navigate in the face of a changing, or nonstationary environment.
Following \citet{ucbexp}, we consider the case when an agent interacts with a controlled Markov process (CMP) %
%\footnote{Readers familiar with Markovian decision process (MDP) may note that a controlled Markov process is simply an MDP with the rewards removed.}
 equipped with finitely many actions and at most countably many states, the state is observable after every transition and a reset action is available which brings the agent back to some initial state. The problem then is to minimize the number of steps where the agent lacks the ability to reliably navigate to safely reachable states. Since the number of states is unbounded, the agent is given as input a `radius' $L$ such that it needs to consider all states that are reachable within $L$ steps (precise definitions will be given in the next section). 
\citet{ucbexp} gave an algorithm that with high probability finishes the discovery task in time that is proportional to the product of $L^3$ and the number of states to be discovered. 
Unlike this previous work, we consider the case \emph{when the transition probabilities can (abruptly) change}. This setting
is important as agents with a long ``lifespan'' may expect their environment to change: 
``moving parts'' can suddenly break down as commonly experienced in robotics or more generally in automation \citep{doi:10.1177/0278364913495721}, or the environment may change abruptly due to the appearance or disappearance of other agents, or objects, such as in rescue robots in urban search and rescue missions in unknown environments  \citep{8606991}. 
The time when the changes happen or the nature of the changes are unknown. In this new setting, we consider the problem of minimizing the number of exploration steps: A time step is considered an exploration step if at that time step the agent % TODO: agent or learner
lacks sufficient knowledge to navigate its current environment efficiently. The challenge is of course that the agent may not be aware of when it does not have this sufficient knowledge. For this problem we give a meta-algorithm \alg\ which can utilize any base algorithm designed for the stationary version of the problem and which keeps the number of exploration steps below $O(F^2)$ when the number of environment changes is $F$.
Changing environments have been studied in the context of reinforcement learning 
%and Markov Decision Processes 
(see e.g., 
%\citet{Nilim:2005:RCM:1246500.1246504}, \citet{DBLP:conf/nips/XuM06}, 
\citet{NIPS2004_2730,
%\citet{Yu2009a}, \citet{Yu2009b}, 
NIPS2013_4975, 
%\citet{Dick2014}, 
UAI2019}). However, our problem setting fundamentally differs from these works as the external rewards are absent and as such our performance metric is incomparable. 
\section{Problem Setting}
\label{sec:Setting}
We consider a discrete-time controlled Markov process -- a Markov decision process where rewards are absent. We assume a countable, possibly infinite state space $\sS$ and a finite action space $\sA$ with $A=|\sA|$ actions. Upon executing an action $a \in \sA$ in state $s \in \sS$ at time $t$, the environment transitions into the next state $s' \in \sS$ selected randomly according to the unknown transition probabilities $P_t(s'|s,a)$. 
%\citet{ucbexp} have considered a similar problem setting albeit with stationary transition probabilities. 
In order to define the performance measure for our problem, we make use of some of the preliminary definitions and an assumption from \citet{ucbexp} (Definitions \ref{defn1}--\ref{defn3} and Assumption \ref{assump1} below), which assume $P_t=P$. We assume that the reader is familiar with terminology of Markov decision processes which we borrow from.

The learning agent is expected to solve the autonomous exploration problem in which the goal is to find a policy for each \textit{reachable state} from a starting state $s_0$, which we will fix for the rest of the article, and hence will be omitted from any notation.
\begin{definition}[Navigation time]
\label{defn1}
For any (possibly non-stationary) policy $\pi$, let $\tau(s|\pi)$ be 
the expected number of steps before reaching~$s$ for the first time when executing policy $\pi$ starting from $s_0$. 
\end{definition}

The learner will be given a number $L>0$ and we may naively demand that it finds all states reachable in at most $L$ steps:
\begin{definition}[$S_L$]
\label{defn2} 
We let $S_L$ denote %the set of all the states reachable in at most $L$ steps:
$
S_L \defined  \{s \in \sS : \min_{\pi}\tau(s|\pi) \leq L\}.
$
\end{definition}
Since the state space might be infinite, a learner could wander off in some direction or get stuck without being able to return to the starting state. To exclude this possibility, we make the following assumption. 
\begin{assumption}
\label{assump1}
In every state, there is a designated $\reset$ action available, that will transition back to the starting state $s_0$ with probability $1$.
\end{assumption}

We define a {\em policy $\pi$ on $\sS' \subset \sS $} to be a policy with~$\pi(s)=\reset$ for any~$s \not\in \sS'$.
As it turns out, in general it is too much to ask for learners to discover all the states in $S_L$. Rather, following \citet{ucbexp} we require learners to discover only the so-called incrementally discoverable states, $\sSa_L$.  %\citet{ucbexp} shows that this set can be learned efficiently:
\begin{definition}[$\sSa_L$]
\label{defn3}
Let $\prec$ be some partial order on $\sS$. The set $\sS^\prec_L$ of states reachable in $L$ steps with respect to ~$\prec$, is defined inductively as follows:
\begin{itemize}
\item $s_0 \in \sS^\prec_L$,
\item if there is a policy $\pi$ on $\{s' \in \sS^\prec_L: s' \prec s\}$ with $\tau(s|\pi) \leq L$, then~$s \in \sS^\prec_L$.
\end{itemize}
The set $\sSa_L$ of states {\em reachable in $L$ steps in respect to some partial order} is given by 
$ \sSa_L \defined \bigcup_\prec \sS^\prec_L $, where the union is over all possible partial orders.
\end{definition}

%For meaningful results in our setting, the number of times the CMP setting changes also needs to be taken into account.  
Back to the nonstationary case, 
we define the number of changes in the environment as
\[
F \defined \#\{1 \leq t | \exists s', s, a : P_{t-1}(s'|s,a) \neq P_t(s'|s,a) \}.
\]
For notational convenience, we assume that $P_0(s'|s,a) \neq P_1(s'|s,a)$ for some $(s,a)$, thereby always counting the first change at $t=1$. Therefore,
\begin{equation}
\label{def:F}
\text{\#changes} = \text{\# different CMP settings} = F.     
\end{equation}

Next we define the performance measure we propose for the considered problem setting. 

\begin{definition}[Exploration steps] 
\label{def:expsteps}
The $(L,\epsilon)$-exploration steps are the complement of the set
$\sNExp$, where $\sNExp$ contains the time steps $t$ at which the learner
\begin{itemize}
    \item has identified a set $\sK \supseteq \sSa_L$ for the CMP with transition probabilities $P_t$, and
    \item has a policy $\pi_s$ for every state $s \in \sK$ with $\tau(s|\pi_s) \leq (1 + \epsilon)L$ for the transition probabilities $P_t$.
\end{itemize} 
%The set of exploration steps is defined as the complement of $\sNExp$. 
\end{definition}
The set of exploration steps contains the time steps for which the learner doesn't have \textit{sufficient} knowledge about the current CMP structure in order to navigate to reachable states from $s_0$ efficiently. The learner's aim is to be able to efficiently navigate the current CMP structure at most of the time steps, or equivalently to minimize the number of exploration steps. 

\textbf{Introduction to \ucbexp \citep{ucbexp}}:
Before we illustrate our meta-algorithm using \ucbexp\ as a subroutine, let us take a look at a few relevant details. \ucbexp \ alternates between two phases: \textit{state discovery} and \textit{policy evaluation}. In a state discovery phase, new candidate states are discovered as potential members of the set of reachable states. In a policy evaluation phase, the optimistic policy $\pi_s$ for reaching one of the candidate states $s$ is evaluated to verify if $\pi_s$ is \textit{acceptable}\footnote{By an acceptable policy, we mean any policy $\pi_s$ such that $\tau(s|\pi_s) \leq (1+\epsilon)L$.}.
%\footnote{The details of which candidate state is selected and how the optimistic policy is found are not pertinent here.} 
A policy evaluation phase for any $\pi_s$ lasts for a certain number of episodes. Each episode begins at $s_0$ and ends either when $\pi_s$ successfully reaches $s$ or $\left\lceil \left( 1 + \tfrac{1}{\epsilon}\right)L \right\rceil$ steps have been executed. If $s$ is not reached in a suitably high number of episodes, policy evaluation for $\pi_s$ is said to have \textit{failed}. A successful policy evaluation means a new reachable state and an acceptable policy have been discovered. A failed policy evaluation leads to selecting another candidate state-optimistic policy pair for evaluating while a successful policy evaluation leads to a state discovery phase which in turn adds more candidate states for the subsequent policy evaluation phases. 
%\ucbexp terminates when 
We restate the main result of \citet{ucbexp} below for reference. % as it will be used later. 
\begin{theorem}\label{thm:ref}[\citet[Theorem 8]{ucbexp}]
When algorithm \ucbexp \ is run on a stationary CMP problem (i.e $\forall t, P_t(\cdot|s,a) = P(\cdot|s,a) $) with inputs $s_0$, $\mathcal{A}$, $L\geq 1$, $\varepsilon>0$, and $\delta$, then with probability $1-\delta$ 
\begin{itemize}
\item it discovers a set of states $\mathcal{K} \supseteq \sSa_L$;
\item  for each $s \in \mathcal{K}$, it outputs a policy $\pi_s$ with $\tau(s|\pi_s) \leq (1+\varepsilon)L$, and
\item it terminates after $O\left(\frac{SAL^3}{\varepsilon^3}\left(\log\frac{SAL}{\varepsilon\delta}\right)^3 \right)$
exploration steps,
where $S=|\mathcal{K}| \leq |\sSa_{(1+\varepsilon)L}|$.
\end{itemize}
\end{theorem}
\section{Meta-algorithm for autonomous exploration in non-stationary CMPs} 
\label{sec:Alg}
\begin{figure*}
\centering
\fbox{
\begin{minipage}{\textwidth}
\begin{multicols}{2}
\begin{algorithmic}[1]
\Statex \textbf{Input}:
A confidence parameter $\delta$,
an error threshold $\varepsilon>0$, $L\geq 1$,
$\mathcal{A}$, $s_0$, constants $C_1>0$ and $C_2>0$.
%\Statex \textbf{Initialization}: Initiate round number $r=1$.
%\Statex For round $r$,
\Statex For round $r=1,2, \dots$
\Statex \textbf{Building phase:}
        %\begin{enumerate}
            %\item Initialize $i=0$.  
            %\item Run \ucbexp \  for up to $T_i$ time steps. If it stops $\leq T_i$, note the set $K$ and the respective policies and go to phase 2.
            %\item Otherwise $i \gets i + 1$ and go to step 2 of phase 1. 
            %\item Set $t_r = t$ to be the beginning of the current round. 
           
                 %\State Set the start time of the current round $t_r=t$.
                 \State Initialize $\sStreams = \{\}$. The set $\sStreams$ is used to store the set of initiated streams in round $r$ so far.
                 % \item Set $s=1$. This
                 \State \label{step:stream}\textbf{Stream handling:} Let $q_r$ indicate the current quantum of time steps within the building phase of round $r$. The length of $q_r$ is determined dynamically (explained below in step \ref{step:lenghtq}) but is at most $\left\lceil \left( 1 + \tfrac{1}{\epsilon}\right)L \right\rceil$.  
                 %Since, we use \ucbexp \ as a subroutine, $|q_{.}| = \left\lceil \left( 1 + \tfrac{1}{\epsilon}\right)L \right\rceil$ 
                 Let $\delta'_r = \tfrac{3 \delta}{4 \pi^2 r^2}$. 
                 \Statex For $q_r=1,2, \dots$ 
                        \begin{enumerate}[(a)]
                            \item \label{step:init} \textbf{Initiation rule:} For any integer $\stream\geq1$, if $q_r = (\stream-1)^2 + 1$, then add $\stream$ to $\sStreams$. Initiate a new copy of \ucbexp$(\delta'_r,  \epsilon, L, \sA, s_0)$ and associate it with stream $\stream$. This copy of \ucbexp \ acts only according to the samples taken from the time steps at which $p$ is active.
                            \State \label{step:alloc} \textbf{Allocation rules:}
                                \begin{enumerate}[label=(\roman*)]
                                \item \label{step:alloc1} If $q_r = 1$, activate the only initiated stream in $\sStreams$ so far i.e. $\stream_{q_r} = 1$.
                                \item \label{step:alloc2} Otherwise if all the initiated streams in $\sStreams$ have been active for equal number of quantums previously,  then 
                                $\stream_{q_r} = $ least recently active stream in $\sStreams$.
                                \item \label{step:alloc3} Otherwise $\stream_{q_r} = $ the stream in $\sStreams$ which has been active for the least number of quantums previously. 
                                \end{enumerate}
                            %\item If $q_r > \left( \lceil \sqrt{q_r} \rceil \right)^2 - \lceil \sqrt{q_r} \rceil$, \\
                            %\hspace*{5mm} allocate quantum $q_r$ to stream to an initiated stream $\stream = \lceil \sqrt{q_r} \rceil $, \\
                            %else, \\
                            %\hspace*{5mm} allocate quantum $q_r$ to an initiated  stream $\stream = q_r - \lfloor \sqrt{q_r} \rfloor^2$.  
                            %\item Let the active stream The stream $\stream$ may be in either state discovery stage or policy evaluation stage. 
                            %If $\stream$ is in state discovery stage, $|q_r| \gets \lceil (1+ \tfrac{1}{\epsilon})L\rceil$, otherwise $\stream$ in policy evaluation. On the other hand if $\stream$ is in policy evaluation stage, run an episode of the the copy of \ucbexp \  in stream $\stream$, which takes at most $\lceil (1+ \tfrac{1}{\epsilon})L\rceil$ time steps and therefore $|q_r| \leq \lceil (1+ \tfrac{1}{\epsilon})L\rceil$. In both the cases, the active stream $\stream$ for the current quantum is run for $|q_r|$ time steps. 
                            \State \label{step:lenghtq} %Run the copy of \ucbexp \  for the activated stream in the quantum $q_r$.
                            If the copy of \ucbexp \ associated with the stream $\stream_{q_r}$ is in a state discovery phase, run it for $\left\lceil \left( 1 + \tfrac{1}{\epsilon}\right)L \right\rceil$ time steps. Otherwise the copy of \ucbexp \ associated with the stream $\stream_{q_r}$ is in a policy evaluation phase, and then run it for an episode \big(which is always $\leq \left\lceil \left( 1 + \tfrac{1}{\epsilon}\right)L \right\rceil$ time steps\big) of policy evaluation of \ucbexp \ 
                            i.e.,
                            $$
                            |q_r| = 
                            \begin{cases}
                              \left\lceil \left( 1 + \tfrac{1}{\epsilon}\right)L \right\rceil, \text{if $\stream_{q_r}$ in state discovery} \\
                              |\text{episode of policy evaluation of $\stream_{q_r}$}|%, \text{otherwise}
                            \end{cases}
                            $$
                            \State \label{step:buildEnd} \textbf{Check for the end of building phase:} If during $q_r$, the copy of \ucbexp \ associated with the active stream terminates and provides a set of reachable states and acceptable policies for them, record them in $\sK_r$ and $\sP_r = \{ \pi_s, \forall s \in \sK_r \}$ respectively. Terminate all the other initiated streams in $\sStreams$ and proceed to the checking phase. Otherwise proceed to next $q_r$.
                        \end{enumerate}
        %\end{enumerate}
    \Statex \textbf{Checking phase:} 
       \Statex Compute $W_r = \tfrac{ C_1 |\mathcal{K}_r| AL^3}{\epsilon^3} \left( \log{  \tfrac{ C_2 |\mathcal{K}_r|AL}{\epsilon \delta'_r}} \right)^3$ and $\checkterm_r = \sqrt{\frac{ \log(1/ \delta'_r)}{ 2 \left( \log\left( |\sK_r|AL/\epsilon \delta'_r\right) \right)^3}}$. Let a single check-run consist of the following two parts in the given order: a new copy of \ucbexp$(\delta'_r, \epsilon, L, \sA, s_0)$ running for up to $W_r$ time steps and a policy evaluation phase of \ucbexp \ for each of the policies in $\sP_r$. 
       If the first part of any check-run doesn't terminate within $W_r$ time steps, then terminate it manually and proceed to the second part of the check-run. Set $n_r = \left( \log{ \tfrac{|\sK_r| AL}{\epsilon \delta'_r}} \right)^3$. Execute $n_r$ check-runs. Then: 
            \State\label{Start:checking} Let $\de$ be the number of times \ucbexp \  has failed to terminate within $W_r$ time steps in the first part during the last $n_r$ check-runs. If
            \begin{equation}\label{eq:test1}
                \frac{\de}{n_r} > \checkterm_r + \delta'_r, 
            \end{equation}
            then stop the checking phase, set $r \gets r +1$ and start a new round, otherwise proceed to next step. 
            %%%%%%%%%%%%
            \State For every state $s$ in $\mathcal{K}_r$, let $\de'_s$ be the number of times policy evaluation fails for $\pi_s \in \sP_r$ in the second part of the last $n_r$ check-runs. If
            \begin{equation}\label{eq:test2}
                \frac{\de'_s}{n_r}  > \checkterm_r + \delta'_r,
            \end{equation}
            delete $s$ and $\pi_s$ from $\sK_r$ and $\sP_r$ respectively. Proceed to next step.
             \State Let $s$ be any state which was absent in $\mathcal{K}_r$, but has appeared in the output of at least one of the first part of the last $n_r$ check-runs. For every such state $s$, let $\vic_s$ be the number of times $s$ was present in the output of the first part of the last $n_r$ check-runs. If
            \begin{equation}\label{eq:test3}
            \delta'_r - \left(1 - \frac{\vic_s}{n_r}\right) >  \checkterm_r ,
            \end{equation}
            add $s$ and the last found policy for $s$ to $\sK_r$ and $\sP_r$ respectively. Proceed to next step. 
            \State Execute a check-run one more time. 
            Go back to step \ref{Start:checking} of checking phase. 
\end{algorithmic}
\end{multicols}
\end{minipage}
}
\caption{\underline{M}eta-algorithm for autonomous exploration in \underline{n}on-stationary C\underline{M}Ps or \alg}
\label{fig:alg}
\end{figure*}
%Let $T_i = \lfloor T_0 b^i \rfloor $ for $b \in \mathbb{R}$, $b>1$ and $T_0 \in \mathbb{N}$ 
Our meta-algorithm (\underline{M}eta-algorithm for autonomous exploration in \underline{n}on-stationary C\underline{M}Ps or \alg) can use any algorithm designed for autonomous exploration in a stationary CMP as a \emph{subroutine}. In Figure 1, for the sake of specificity, we describe the algorithm using \ucbexp \ \citep{ucbexp} as a subroutine. 

The algorithm proceeds in rounds and each round consists of two phases: a building phase and a checking phase. In a building phase, we build a hypothesis which consists of a set of states and an acceptable policy for each of them. In a checking phase, we check if the hypothesis we built in this round is still valid. In any building phase, the algorithm initiates several copies of the subroutine at different time steps (see \ref{step:stream}\ref{step:init} in Figure \ref{fig:alg}) and switches back and forth between them (see \ref{step:stream}\ref{step:alloc} in Figure \ref{fig:alg}). Once it switches to a copy of the subroutine, that subroutine is said to be to active and it remains so until the next switch. To simulate this approach, our algorithm proceeds in \textit{streams}. A stream is a single run of the subroutine acting only according to the previous time-steps for which the said stream is active. At any time step, only a single stream is active. Once a stream is active it stays so for a quantum of time steps, the length of which is determined dynamically (see \ref{step:stream}\ref{step:lenghtq} in Figure \ref{fig:alg}). When a hypothesis is formed in the building phase of a round $r$, it is stored in $\sK_r$ and $\sP_r$ (see \ref{step:stream}\ref{step:buildEnd} in Figure \ref{fig:alg}) and the algorithm moves on to the checking phase. 

In the checking phase, recent history is examined, by employing a sliding window, to detect various kinds of changes in the hypothesis. When the hypothesis is found to be valid no more on account of a change, the algorithm terminates the checking phase and proceeds to the next round. In the checking phase, our algorithm employs the subroutine as a black-box using the upper bound on the exploration steps required by the subroutine for a stationary CMP problem. 
Using Theorem \ref{thm:ref}, the upper bound is $O(\tfrac{SAL^3}{\epsilon^3} \left( \log{\tfrac{SAL}{\epsilon\delta}} \right)^3)$ for \ucbexp. We use this bound to compute $W_r$ for each round $r$ with suitable constants $C_1$ and $C_2$.  

\begin{comment}
The subroutine i.e \ucbexp \ is used to detect any changes in the set of reachable states or changes which cause a previously acceptable policy for a still reachable state to become unacceptable. For the latter, we employ only the policy evaluation phase (step 2(c) of Figure 1 from \cite{ucbexp}) of \ucbexp \ for the policies in $\sP_r$. Each episode within the policy evaluation phase of \ucbexp \ begins at $s_0$ and ends either when the policy being verified $\pi_s \in \sP_r$ successfully reaches $s$ or $\left\lceil \left( 1 + \tfrac{1}{\epsilon}\right)L \right\rceil$ steps have been executed. If $s$ is not reached in suitably high number of episodes, policy evaluation for $\pi_s$ is said to have \textit{failed}. 
\end{comment}

 At any time step $t$, our algorithm's knowledge of the current CMP structure is represented by  $\sK_{r-1}$ and $\sP_{r-1}$ where $r$ the current round at $t$. When the current round $r=1$, the algorithm is yet to learn the present CMP structure. 

\alg \ can use any algorithm designed for autonomous exploration in a stationary CMP as a subroutine if it is provided with two values:
\begin{itemize}
    \item the length of the quantum i.e. the number of contiguous time steps for which a copy of the subroutine (i.e., a stream) must be active, and 
    \item a high-probability upper bound on the number of exploration steps required by the subroutine for a stationary CMP problem.
\end{itemize}
These two values are used in  Step \ref{step:stream}\ref{step:lenghtq} and the computation of $W_r$ at the beginning of a checking phase respectively (see Figure \ref{fig:alg}). Using another algorithm as a subroutine instead of \ucbexp \ would only cause these two changes with the rest of \alg \ remaining the same. 

Our main result, stated in Theorem \ref{thm:main}, upper-bounds the number of exploration steps required by \alg \ using \ucbexp \ as a subroutine. The corresponding result while using other subroutines could simply be obtained by replacing the upper bound of exploration steps required by \ucbexp \ for a stationary CMP with the analogous bound of the subroutine being used. 

\begin{theorem}
\label{thm:main}
With probability $1 - \delta$, the total number of exploration steps for \alg \ using \ucbexp \ as a subroutine and with inputs $s_0$, $L \geq 1$, $\epsilon$, $\delta$, $C_1 = 216\cdot (15)^2 + 61$ and $C_2 = 225$  is upper bounded by 
\begin{align*}
&\left(\sum_{f=1}^{F} \frac{ C_1 S_f A L^3}{\epsilon^3} \left( \log\frac{4 \pi^2 C_2 F^2 S_f AL}{3 \epsilon \delta} \right)^3 \right)^2   \\
%%%%%%%%%%%%%
& + F \max_{f \in \{1, \dots, F\}} \left[\frac{ 2 C_1 S_{f} A L^3}{\epsilon^3} \left( \log\frac{4 \pi^2 C_2 F^2 S_{f} AL}{3 \epsilon \delta} \right)^6 \right], 
\end{align*}
 where $S_f = |S^{\rightarrow}_{(1+\epsilon)L}(f)|$ is the number of incrementally discoverable states reachable in $(1+ \epsilon)L$ time steps in the $f^\text{th}$ CMP setting, and $\#$changes = $F$.
 %and $\sF$ is the set of all underlying CMP settings with $|\sF| = F$
\end{theorem}
Note that a change in this context affects the set of reachable states in $(1 + \epsilon)L$ steps from $s_0$ and/or the acceptable policies for reaching them. The reason, as noted by \citet{ucbexp}, is that the learner cannot distinguish between the states reachable in $L$ steps and those reachable in $(1 + \epsilon)L$ steps (given a reasonable amount of exploration). 

\textbf{Motivating factors for the construction of our algorithm}
\begin{itemize}
\item Before an algorithm forms a hypothesis i.e., it determines a set of reachable states and acceptable policies, it might not be possible to detect a change. Consider an algorithm still in the process of building a hypothesis. During this process, the algorithm must proceed and inspect states in some order. Suppose that it has found acceptable polices for some reachable states. When it finds a new reachable state, there are two plausible scenarios: \emph{a)} this state was not reachable when the algorithm was in the process of inspecting other states earlier i.e., there was a change, or \emph{b)} this state was reachable when the algorithm was in the process of inspecting other states earlier i.e., there was no change. It is not possible to distinguish between these two scenarios. 
\item Since it might not be possible to detect a change during the hypothesis building phase and a change can occur at any time, the algorithm needs to start several processes during the hypothesis building phase. Each process aims to form a hypothesis for a particular CMP setting and to be able to do that, it needs to act only on the time steps for which that CMP setting is in effect. On one hand, since a change can occur at any time step, the algorithm needs to start these processes at several time steps along the way. On the other hand, if too many processes are started, each process will not get enough time to form its hypothesis. Therefore sufficient time should be allocated to each process. Both of these diverging requirements can be balanced, if both the number of processes and the time allocated to each process grow asymptotically as the square-root of time, as done in \alg.

\item \textbf{Using a sliding window in the building phase is not possible:}
Using a sliding window in the checking phase is possible as each check-run is verifying the same hypothesis (the one found in the preceding building phase) and hence findings from successive check-runs can be shared. In the building phase however, each stream with its own copy of the subroutine might be attempting to build different hypotheses and hence findings from two different streams cannot be shared.  
\end{itemize}

\section{Analysis and Proof of Theorem \ref{thm:main}}
\label{sec:Proofs}
First, we bound the number of exploration steps in a single building phase. Then we prove that the number of rounds is upper-bounded by $\#$changes $F$. Combining these two, we prove an upper bound on the number of exploration steps for all the building phases. Next, we prove an upper bound on the number of exploration steps in a checking phase caused by a single change. Summing this over all the changes in all the rounds gives us an upper bound on the number of exploration steps in all the checking phases. Finally, we add the respective upper bounds for all the building phases and all the checking phases to arrive at the bound given by Theorem \ref{thm:main}. 

\subsection{Bounding the exploration steps in a single building phase}
First, we state a couple of preliminary lemmas about stream handling to be used later. 
\begin{lemma}
\label{lem:StreamCount}
At the end of any quantum $q$ in a building phase, the number of initiated streams is equal to $\lceil \sqrt{q} \rceil$.
\end{lemma}
\iftoggle{long-version}{The proof for Lemma \ref{lem:StreamCount} is given in Appendix \ref{Proof:StreamCount}.}{Please consult the extended version of this article for the complete proof of this Lemma.}
Here, we provide a brief overview. We use the fact that the $\#$initiated streams is equal to the highest stream number initiated so far and the stream initiation rule (\ref{step:stream}\ref{step:init} in Figure \ref{fig:alg}) to arrive at this claim. 
\begin{comment}
\begin{proof}
The number of initiated streams is equal to the highest stream number initiated so far. Let that be $\hat{\stream}$.
%As described in the algorithm a stream $\stream$ is initiated in the quantum $(\stream-1)^2 + 1$. 
Since $\hat{\stream}$ is initiated on or before $q$, $(\hat{\stream} -  1)^2 + 1 \leq q$ (see \ref{step:stream}\ref{step:init} in Figure \ref{fig:alg}) which is equivalent to, 
\begin{equation}
    %(\hat{\stream} -  1)^2 + 1 &\leq q \nonumber\\
    %%%%%%%%%%%%%%%%%%%%
    %(\hat{\stream} -  1)^2 &< q \nonumber \\
    %%%%%%%%%%%%%%%%%%%%
    \hat{\stream} < \sqrt{q} + 1. \label{eq:StreamCount1}
\end{equation}
Since $\hat{\stream} + 1$ has not been initiated yet, $q \leq (\hat{\stream} + 1 -1)^2$ which translates to,
\begin{equation}
    \hat{\stream} \geq \sqrt{q} \label{eq:StreamCount2}.
\end{equation}
    Recall that both $\hat{\stream}$ and $q$ are integers $\geq$1. If $q$ is a perfect square, the only integer satisfying both Eq. $\eqref{eq:StreamCount1}$ and $\eqref{eq:StreamCount2}$ is $\sqrt{q} = \lceil \sqrt{q} \rceil$. If $q$ is not a perfect square, then Eq. $\eqref{eq:StreamCount2}$ reduces to $\hat{\stream} > \sqrt{q}$. And the only integer satisfying $\sqrt{q} < \stream < \sqrt{q} + 1$ is $\lceil \sqrt{q}\rceil$.
 \end{proof}   
\end{comment}

\begin{lemma}
\label{lem:StreamAlloc}
At the end of a quantum $q = b^2$ for some integer $b \geq 1$,
\begin{enumerate}
    \item $b$ streams have been initiated, and
    \item each initiated stream has been active for exactly $b$ quantums.   
\end{enumerate}
\end{lemma}
\iftoggle{long-version}{The proof for Lemma \ref{lem:StreamAlloc} is given in Appendix \ref{Proof:StreamAlloc}.}{We refer to the extended version of this article for the detailed proof.} We only provide a proof sketch here.
%\textit{Proof Sketch}:  
Claim 1 is a direct result of Lemma \ref{lem:StreamCount}. Claim 2 can be proved by induction on $b$ and considering the initiation rule and allocation rules (ii) and (iii) (see \ref{step:stream}\ref{step:init}, \ref{step:stream}\ref{step:alloc}\ref{step:alloc2} and \ref{step:stream}\ref{step:alloc}\ref{step:alloc3} in Figure \ref{fig:alg} respectively). 
\begin{comment}
\begin{proof}
    Claim 1 is a direct result of Lemma \ref{lem:StreamCount}. We prove claim 2 by induction on $b$. 
    Base case: $b =1$. At the end of $q = b^2 = 1$, only $1$ stream has been initiated and it has been active for $1$ quantum. \\ 
    Inductive case: Let's assume that the claim is true for $b = \hat{b}$ i.e at the end of quantum $q = \hat{b}^2$, exactly $\hat{b}$ streams have been initiated and each of them has been active for $\hat{b}$ quantums. At the next quantum i.e $\hat{b}^2 + 1$, stream $\hat{b} + 1$ will be initiated by the initiation rule and it will be active for the next $b$ quantums due to the allocation rule (ii). At this point, we are at the end of quantum $(\hat{b}+1) \cdot \hat{b}$ and
    all the $\hat{b}+1$ initiated streams have each been active for $\hat{b}$ quantums. Next, by virtue of the allocation rule (iii), each of the $(b+1)$ streams will be allocated $1$ quantum each till we are the end of quantum $ ((\hat{b}+1) \cdot \hat{b}) + (\hat{b}+1) = (\hat{b}+1)^2$. 
\end{proof}
\end{comment}

\begin{lemma}
\label{lem:buidlingphase}
In a round $r$, with probability at least $1 - \delta'_r$,
\begin{enumerate}
    \item the length of the building phase is at most $$
 \left(\sum_{\mdp \in \smdp^b_r} \frac{ C_1 S_{(\mdp)} A L^3}{\epsilon^3} \left( \log\frac{C_2 S_{(\mdp)} AL}{\epsilon \delta'_r} \right)^3 \right)^2 ,
$$
    \item the building phase discovers a set of states $\mathcal{K}_r \supseteq \sSa_L(\hmdp_r)$ for some CMP setting $\hmdp_r \in \smdp^b_r$ (the set of underlying CMP settings during the building phase of round $r$),
    \item and for each $s \in \mathcal{K}_r$, $\sP_r$ contains a policy $\pi_s$ with $\tau(s|\pi_s) \leq (1+\varepsilon)L$ for the CMP setting $\hmdp_r \in \smdp^b_r$,
\end{enumerate}
where $S_{(m)} = |S^{\rightarrow}_{(1+\epsilon)L}(\mdp)|$ is the number of incrementally discoverable states reachable in $(1+ \epsilon)L$ time steps in the CMP $\mdp$, $C_1 = 216\cdot (15)^2 + 61$ and $C_2 = 225$.
\end{lemma}
\begin{proof}
Consider the CMP setting $m_{r,1}$ at the start of the building phase in round $r$. Assume that \ucbexp \ requires at most $x_{r,1}$ quantums as exploration steps for $m_{r,1}$ (without any change) with high probability. Theorem \ref{thm:ref} shows that the exploration steps for the CMP setting $m_{r,1}$ required by \ucbexp \ are at most $\tfrac{C_1 S_{(m_{r,1})} A L^3}{\epsilon^3} \left( \log\tfrac{C_2 S_{(m_{r,1})} AL}{\epsilon \delta'_r} \right)^3$ with high probability. There are two possible cases:
    
        \textbf{Case 1}: The problem doesn't change for the duration of $x_{r,1}^2$ quantums. \\ 
        Our meta-algorithm initiates stream 1 at $q=1$ and this stream will have been active for $x_{r,1}$ quantums at the end of quantum $x_{r,1}^2$ (using Lemma \ref{lem:StreamAlloc}). Since the problem doesn't change for this entire duration, the copy of \ucbexp \  for stream 1 has samples only of $m_{r,1}$. Thus, stream 1 terminates at the end of $x_{r,1}^2$ quantums of our meta-algorithm with probability $1- \delta'_r$ and the building phase of round $r$ ends. 
        The three claims of the lemma follow from the respective claims\footnote{The referred theorem only mentions the upper bound in terms of big-O notation. However, the constants $C_1$ and $C_2$ can be computed from its proof in \citet[Section 4.6]{ucbexp}.} of Theorem \ref{thm:ref} with $\smdp_r^b = \{m_{r,1}\}$. 
        
        \textbf{Case 2}: The problem changes at any point before the end of $x_{r,1}^2$ quantums. \\
        Let $m_{r,1}$, $m_{r,2}, \dots$ be the successive CMP settings. Let $x_{r,i}$ be the required number of quantums needed by \ucbexp \  for each $m_{r,i}$. Let $m_{r,k}$ be the first problem setting which doesn't change from quantum $(x_{r,1} + \dots + x_{r,k-1})^2 + 1$ to quantum $(x_{r,1} + \dots + x_{r,k})^2 $. 
        %Such a $k$ exists because of the Pigeonhole principle. 
        The stream $x_{r,1} + \dots + x_{r,k-1} + 1$ starting at $(x_{r,1} + \dots + x_{r,k-1})^2 + 1$ will have been active for $x_{r,k}$ quantums at the end of quantum $(x_{r,1} + \dots + x_{r,k})^2$ (using Lemma \ref{lem:StreamAlloc}). That stream will therefore terminate and output the set of reachable states and acceptable policies for $m_{r,k}$ at the end of quantum $(x_{r,1} + \dots + x_{r,k})^2$ with probability $1 - \delta'_r$ and the building phase will terminate. The three claims of the lemma follow from the respective claims of Theorem \ref{thm:ref} with $\smdp_r^b = \{m_{r,1}, \dots, m_{r,k}\}$. 
    
\end{proof}

\subsection{Bounding the number of rounds}
%\subsection{Analyzing the checking phase}
\begin{lemma}
\label{lem:numrounds2}
With probability at least $1- \delta/4$, the total number of rounds $ R \leq F$.
\end{lemma}
\begin{proof}
There is always at least one change in the building phase of the first round as the first change is counted at $t=1$ by default (see Section \ref{sec:Setting}). For $1<r<R$, we consider the following two mutually exclusive cases.

\textbf{Case 1:} There exists no round $1 < r < R$ which has no change in its building phase. \\
        In this case, every round contains at least one change and the total number of rounds is immediately upper-bounded by $F-1$. 
        
\textbf{Case 2:} There exists at least one round  $1 < r < R$ which has no change in its building phase. \\
    Let $r$ be a round such that there is no change during its building phase. For all such rounds $r$ which contain no change in the building phase, with probability at least $1- \delta/4$, we prove that the checking phase of $r$ contains at least one change. 

    Recall from Lemma \ref{lem:buidlingphase} that the sole CMP setting during the building phase of round $r$ is denoted as $\hmdp_r$ and the building phase discovers the reachable states for $\hmdp_r$ with probability at least $1 - \delta'_r$. 
    %Let $\sK_r$ be the set of states output by this building phase.  
    Theorem \ref{thm:ref} shows that for the CMP setting $\hmdp_r$, if \ucbexp \  with run for $W_r = \tfrac{ C_1 |\mathcal{K}_r| AL^3}{\epsilon^3} \left( \log{\tfrac{ C_2 |\mathcal{K}_r|AL}{\epsilon \delta'_r}} \right)^3$ steps, then the failure probability (i.e., the probability with which \ucbexp \  doesn't terminate at the end of at most $W_r$ time steps) is at most $\delta'_r$ where $C_1 = 216\cdot (15)^2 + 61$ and $C_2 = 225$.

\noindent The only condition to trigger the next round is given by Eq. \eqref{eq:test1}. Therefore, when round $r$ ends, 
$$
      \frac{\de}{n_r} > \checkterm_r + \delta'_r, 
$$    
where $\de$ is the number of times \ucbexp \  has failed to stop and return a set of reachable states within $W_r$ time steps during the first part of the last $n_r$
%=\left( \log{ \tfrac{|\mathcal{K}_r| AL}{\epsilon \delta_r}} \right)^3$ 
check-runs.
%and $\checkterm_r = \sqrt{\frac{ \log(1/ \delta'_r)}{ 2 \left( \log\left( |\sK_r|AL/\epsilon \delta'_r\right) \right)^3}}$.
If the CMP setting was indeed $\hmdp_r$ (i.e. there was no change) in the last $n_r$ check-runs, then by Hoeffding's inequality,
$$
\bP\left\{ \frac{\de}{n_r} > \checkterm_r +  \delta'_r \right\} \leq \exp(-2\checkterm_r^2 n_r) = \delta'_r.
$$

Therefore when the round $r$ stops, there has been a change in its checking phase with probability at least $(1 - (\delta'_r + \delta'_r))$. Below we use that $\sum_{r=1}^{\infty} \tfrac{1}{r^2} = \tfrac{\pi^2}{6}$.
With a union bound and using $\sum_{r=1}^{R-1} 2 \delta'_r < \tfrac{3\delta}{2\pi^2} \sum_{r=1}^{\infty} \tfrac{1}{r^2} = \frac{\delta}{4}$, we can claim that for all rounds $r < R$ which do not contain a change in the building phase, there is at least one change in each of their respective checking phases with probability at least $1 - \tfrac{\delta}{4}$.

Considering both the cases, we get that the total number of rounds is upper-bounded by the total number of changes $F$ with probability at least $1 - \tfrac{\delta}{4}$.
\end{proof}

\subsection{Bounding the exploration steps in all the building phases}
\begin{lemma}
\label{lem:allbuildingphase}
With probability at least $1 - \tfrac{\delta}{2}$, the total number of exploration steps in all the building phases is at most 
$$
\left(\sum_{f=1}^{F} \frac{ C_1 S_f A L^3}{\epsilon^3} \left( \log\frac{4 \pi^2 C_2 F^2 S_f AL}{3 \epsilon \delta} \right)^3 \right)^2,
$$
where $S_f = |S^{\rightarrow}_{(1+\epsilon)L}(\mdp)|$ is the number of incrementally discoverable states reachable in $(1+ \epsilon)L$ time steps in the $f^\text{th}$ CMP setting and $\#$changes $=F$. 
\end{lemma}
\begin{proof}
We count all the steps in each building phase as exploration steps. Lemma \ref{lem:buidlingphase} provides an upper bound on the number of exploration steps in the building phase of a single round $r$ with the error probability limited to $\delta'_r$. Therefore, the total number of exploration steps in all the building phases is at most
\begin{align*}
& \sum_{r=1}^{R} \left(\sum_{\mdp \in \smdp^b_r} \frac{ C_1 S_{(\mdp)} A L^3}{\epsilon^3} \left( \log\frac{C_2 S_{(\mdp)} AL}{\epsilon \delta'_r} \right)^3 \right)^2  \\
%%%%%%%%%%%%%%%%%%%%%%%%%%%%%%%
& = \sum_{r=1}^{R} \left(\sum_{\mdp \in \smdp^b_r} \frac{ C_1 S_{(\mdp)} A L^3}{\epsilon^3} \left( \log\frac{4 \pi^2 C_2 r^2 S_{(\mdp)} AL}{3 \epsilon \delta} \right)^3 \right)^2 \\
%%%%%%%%%%%%%%%%%%%%%%%%%%%%%%%
& \leq \left(\sum_{f=1}^{F} \frac{ C_1 S_f A L^3}{\epsilon^3} \left( \log\frac{ 4 \pi^2 C_2 F^2 S_f AL}{\epsilon \delta} \right)^3 \right)^2
\end{align*}
with error probability limited to $\left(\tfrac{\delta}{4} + \sum_{r=1}^{R} \delta'_r\right) < \tfrac{\delta}{2}$. 
%Here again we use that $\sum_{r=1}^{\infty} \tfrac{1}{r^2} = \tfrac{\pi^2}{6}$.
In the last inequality, we use that $r \leq R \leq F$ with probability $1 - \tfrac{\delta}{4}$ (Lemma \ref{lem:numrounds2}) and the number of different CMP settings in all the rounds is $F$ (Eq. \eqref{def:F}).
\end{proof}

\subsection{Analyzing the checking phase}
We first bound the number of exploration steps in a checking phase caused due to a single change. 
\begin{lemma}
\label{lem:checkingphase}
With probability $1- \delta'_r$, the total number of exploration steps in the checking phase of a round $r$ due to a single change is at most
$$
 \frac{ 2 C_1 S_{(\hmdp_r)} A L^3}{\epsilon^3} \left( \log\frac{C_2 S_{(\hmdp_r)} AL}{\epsilon \delta'_r} \right)^6  
$$
where $S_{(\hmdp_r)} = |S^{\rightarrow}_{(1+\epsilon)L}({\hmdp_r})|$ is the number of incrementally discoverable states reachable in $(1+ \epsilon)L$ time steps in the CMP setting $\hmdp_r$.
\end{lemma}
\begin{proof}
Recall from Lemma \ref{lem:buidlingphase} that the CMP setting for which the building phase in round $r$ has found the reachable states and acceptable policies is denoted as $\hmdp_r$.   
Below we use that the number of time steps in a single check-run of round $r$ is upper-bounded by $2 \times \tfrac{ C_1 S_{(\hmdp_r)} AL^3}{\epsilon^3} \left( \log{ \tfrac{C_2 S_{(\hmdp_r)}AL}{\epsilon \delta'_r}} \right)^3$ as $|\sK_r| \leq  S_{(\hmdp_r)}$ from Theorem \ref{thm:ref}. Till the CMP setting is $\hmdp_r$ in the checking phase, the algorithm does not incur any exploration steps. 
For a change to $\mdp' \neq \hmdp_r$, the following mutually exclusive and exhaustive cases are possible: \\
\textbf{Case 1:} $\mdp'$ doesn't last for $\left( \log{ \tfrac{|\sK_r| AL}{\epsilon \delta'_r}} \right)^3$ check-runs. \\
    Then all the time steps for which $\mdp'$ is active are considered as exploration steps and they are are upper bounded by
    $$
   \tfrac{ 2 C_1 S_{(\hmdp_r)} AL^3}{\epsilon^3} \left( \log{ \tfrac{C_2 S_{(\hmdp_r)}AL}{\epsilon \delta'_r}} \right)^3 \times \left( \log\tfrac{S_{(\hmdp_r)}AL}{\epsilon \delta'_r} \right)^3.
    $$ 
    \textbf{Case 2:} $\mdp'$ lasts for at least $\left( \log{ \tfrac{|\sK_r| AL}{\epsilon \delta'_r}} \right)^3$ check-runs. \\
        There are three possible subcases. 
        \begin{enumerate}[label=(\alph*)]
            %\textbf{(a) 
            \item $W_r$ time steps are \textit{insufficient} for $\mdp'$. \\
            By insufficient we mean that
            \begin{align*}
            &\tfrac{ C_1 S_{(\hmdp_r)} AL^3}{\epsilon^3} \left( \log{ \tfrac{C_2 S_{(\hmdp_r)}AL}{\epsilon \delta'_r}} \right)^3 \\
            &< \tfrac{ C_1 S_{(\mdp')} AL^3}{\epsilon^3} \left( \log{ \tfrac{C_2 S_{(\mdp')}AL}{\epsilon \delta'_r}} \right)^3
            \end{align*}
            i.e. $ S_{(\hmdp_r)} < S_{(\mdp')}$.
            
            Eq.\eqref{eq:test1} verifies if change to a $\mdp'$ such that $ S_{(\hmdp_r)} < S_{(\mdp')}$ has occurred. Our algorithm keeps a count of the empirical failures in the last $n_r$ check-runs where a failure means that the first part of a check-run has failed to terminate within $W_r$ time steps (and thus had to be manually terminated at $W_r$). From Theorem \ref{thm:ref}, we know that that if no change has occurred then the true failure probability is $\delta'_r$. By Hoeffding's inequality,
            \begin{align*}
                \mathbb{P} \left\{  \tfrac{\de}{n_r} > \checkterm_r + \delta'_r,  \right\} \leq \exp(-2 \checkterm_r^2 n_r) = \delta'_r.
            \end{align*}
            %where $ \checkterm = \sqrt{\frac{ \log(1/ \delta'_r)}{ 2 \left( \log\left( |\sK_r|AL/\epsilon \delta'_r\right) \right)^3}}$. 
            Therefore, with probability $1 - \delta'_r$, we detect a change to $\mdp'$ such that $ S_{(\hmdp_r)} < S_{(\mdp')}$ and the number of exploration steps added are at most
            $$
             2 C_1 \tfrac{S_{(\hmdp_r)} AL^3}{\epsilon^3} \left( \log\tfrac{C_2 S_{(\hmdp_r)}AL}{\epsilon \delta'_r} \right)^3 \times \left( \log\tfrac{S_{(\hmdp_r)}AL}{\epsilon \delta'_r} \right)^3. 
            $$
            %%%%%%%%%%%%%%%%%%%%%%%%%%
            %\textbf{(b) 
            \item \{$W_r$ time steps are \textit{sufficient} for $\mdp'$\} and \{a previously reachable state becomes unreachable in $\mdp'$ or the previously acceptable policy $\pi_s \in \sP_r$ to a reachable state is not acceptable in $\mdp'$\}. \\
            Eq.\eqref{eq:test2} checks for such scenarios. As it keeps verifying if the policy evaluation of $\{ \pi \in \sP_r\}$ succeeds in the last $n_r$ check-runs, it checks for both - i) if a previously reachable state is still reachable and ii) if the previously acceptable policy is still acceptable. 
            Proceeding in a similar manner to the previous subcase, we can show that, with probability  $1 - \delta'_r$, the number of exploration steps added is at most 
            $$
            2 C_1 \tfrac{S_{(\hmdp_r)} AL^3}{\epsilon^3} \left( \log\tfrac{C_2 S_{(\hmdp_r)}AL}{\epsilon \delta'_r} \right)^3 \times \left( \log\tfrac{S_{(\hmdp_r)}AL}{\epsilon \delta'_r} \right)^3.
            $$

            %\textbf{
            \item $W_r$ time steps are \textit{sufficient} for $\mdp'$ and a previously unreachable state becomes reachable in $\mdp'$. \\
            Let's assume that a previously unreachable state $s$ is reachable in $\mdp'$. Either $s \in \sK_r$ or $s \notin \sK_r$. In the former case, policy evaluation (i.e. $2^{\text{nd}}$ part of a check-run) continues to check if $\pi_s \in \sP_r$ is still acceptable. If $\pi_s \in \sP_r$ is found to be acceptable no more, then the check given by Eq. \eqref{eq:test2} will be triggered, the change will be detected and the number of exploration steps added are given by the previous subcase. If $\pi_s \in \sP_r$ is still acceptable, it leads to no additional exploration steps (see Definition \ref{def:expsteps}). 
            Eq. \eqref{eq:test3} checks for scenarios where  $s \notin \sK_r$. 
            Theorem \ref{thm:ref} guarantees that if a state is in $S^{\rightarrow}_L(\mdp')$, the probability that it fails to appear in the output of \ucbexp$(\delta'_r,  \epsilon, L, \sA, s_0)$ is at most $\delta'_r$. For every state $s \notin \sK_r$, but which has appeared in the output of the first part in one of the last $n_r$ check-runs, we can compute the empirical failures as $n_r - v_s$.
            Then, by Hoeffding's inequality
            $$
            \mathbb{P} \left\{\delta'_r - \left(1 - \tfrac{v_s}{n_r}\right)  > \checkterm_r \right\} \leq \exp(-2\checkterm_r^2 n_r) = \delta'_r.
            $$
            Therefore, with probability at least $1 - \delta'_r$, we detect such a change and the number of exploration steps added is at most 
            $$
           2 C_1 \tfrac{S_{(\hmdp_r)} AL^3}{\epsilon^3} \left( \log\tfrac{C_2 S_{(\hmdp_r)}AL}{\epsilon \delta'_r} \right)^3 \times \left( \log\tfrac{S_{(\hmdp_r)}AL}{\epsilon \delta'_r} \right)^3. 
            $$
         \end{enumerate}
    Considering all the cases, the number of exploration steps added is at most $
           2 C_1 \tfrac{S_{(\hmdp_r)} AL^3}{\epsilon^3} \left( \log\tfrac{C_2 S_{(\hmdp_r)}AL}{\epsilon \delta'_r} \right)^6 $ with probability at least $1 - \delta'_r$.
\end{proof}
Now we can bound the number of exploration steps for all the checking phases. 
\begin{lemma}
\label{lem:allcheckingphase}
With probability at least $1 - \tfrac{\delta}{2}$, the total number of exploration steps in all the checking phases is upper-bounded by  
$$
 F \cdot \max_{f \in \{1, \dots, F\}} \left[\frac{ 2 C_1 S_{f} A L^3}{\epsilon^3} \left( \log\frac{4 \pi^2 C_2 F^2 S_{f} AL}{3 \epsilon \delta} \right)^6 \right] 
$$
where $S_f = |S^{\rightarrow}_{(1+\epsilon)L}(f)|$ is the number of incrementally discoverable states reachable in $(1+ \epsilon)L$ time steps in the $f^\text{th}$ CMP setting  and $\#$ changes = $F$. 
\end{lemma}
\begin{proof}
 Lemma \ref{lem:checkingphase} provides an upper bound on the number of exploration steps in the checking phase of a round $r$ due to a single change with error probability limited to $\delta'_r$. 
Due to the construction of our algorithm, only the changes in round $r$ can lead to exploration steps in the checking phase of round $r$. Let $F_r$ be the number of changes in round $r$. Then, the total number of exploration steps are at most
\begin{align*}
    &\sum_{r=1}^{R} \tfrac{ 2 F_r C_1 S_{(\hmdp_r)} A L^3}{\epsilon^3} \left( \log\tfrac{C_2 S_{(\hmdp_r)} AL}{\epsilon \delta'_r} \right)^6   \\
%%%%%%%%%%%%%%%%%%%%%%%%%%%%%%%%%%%
    &\leq \max_{f} \left[\tfrac{ 2 C_1 S_{f} A L^3}{\epsilon^3} \left( \log\tfrac{4 \pi^2 C_2 F^2 S_{f} AL}{3 \epsilon \delta} \right)^6 \right] \cdot \sum_{r=1}^{R}  F_r                           \\
%%%%%%%%%%%%%%%%%%%%%%%%%%%%%%%%%%%    
    & \leq F \cdot  \max_{f} \left[\tfrac{ 2 C_1 S_{f} A L^3}{\epsilon^3} \left( \log\tfrac{4 \pi^2 C_2 F^2 S_{f} AL}{3 \epsilon \delta} \right)^6 \right]
\end{align*}
with error probability limited to $\left(\tfrac{\delta}{4} + \sum_{r=1}^{F} \delta'_r\right) < \tfrac{\delta}{2}$. In the first inequality, we use that $r \leq R \leq F$ with probability $1 - \tfrac{\delta}{4}$ (Lemma \ref{lem:numrounds2}) and  $S_{(\hmdp_r)} \leq \max_{f} S_f$.  
%In the second inequality, we use that the total number of changes is equal to $F$. 
\end{proof}

\subsection{Proof of Theorem \ref{thm:main}}
\begin{proof}
The total number of exploration steps in all the rounds is simply the sum of the exploration steps in all the building phases and all the checking phases given by Lemma \ref{lem:allbuildingphase} and Lemma \ref{lem:allcheckingphase} respectively. Therefore the number of total exploration steps for all the rounds is at most
\begin{align*}
&\left(\sum_{f=1}^{F} \tfrac{ C_1 S_f A L^3}{\epsilon^3} \left( \log\tfrac{4 \pi^2 C_2 F^2 S_f AL}{3 \epsilon \delta} \right)^3 \right)^2   \\
%%%%%%%%%%%%%
& + F \cdot \max_{f \in \{1, \dots, F\}} \left[\tfrac{ 2 C_1 S_{f} A L^3}{\epsilon^3} \left( \log\tfrac{4 \pi^2 C_2 F^2 S_{f} AL}{3 \epsilon \delta} \right)^6 \right] 
\end{align*}
with probability at least $\left(1 - \left(\tfrac{\delta}{2} + \tfrac{\delta}{2}\right)\right)$ using a union bound. 
\end{proof}
%%%%%%%%%%%%%%%%%%%%%%%%%%%%%%%%%%%%%%%%%%%%%%%%%%%%%%%%%%%%%%%%%%%%%%%%%%%%%%%%%%%%%%%%%%%%%%%%
%NEW SECTIOn
%%%%%%%%%%%%%%%%%%%%%%%%%%%%%%%%%%%%%%%%%%%%%%%%%%%%%%%%%%%%%%%%%%%%%%%%%%%%%%%%%%%%%%%%%%%%%%%%
\section{Concluding remarks}
\label{sec:Last}
We considered the problem of learning to explore autonomously in a non-stationary environment and proposed a pertinent performance measure. We gave a natural algorithm for the considered problem and proved an upper bound on the performance measure that scales with the square of the number of changes. 

Proving a lower bound for this problem setting remains for future work.
The solution strategy of first having a building phase (with multiple processes trying to build a hypothesis) and then a checking phase (where it is verified if the last built hypothesis is still true) could be used for other non-stationary learning problems. In particular, this strategy could be useful for the learning problems where each hypothesis building-process needs to act independently and cannot share findings. 

\newpage
%\bibliographystyle{plainnat}
%\bibliography{references} 

\begin{thebibliography}{24}
\providecommand{\natexlab}[1]{#1}
\providecommand{\url}[1]{\texttt{#1}}
\expandafter\ifx\csname urlstyle\endcsname\relax
  \providecommand{\doi}[1]{doi: #1}\else
  \providecommand{\doi}{doi: \begingroup \urlstyle{rm}\Url}\fi

\bibitem[Abbasi et~al.(2013)Abbasi, Bartlett, Kanade, Seldin, and
  Szepesvari]{NIPS2013_4975}
Yasin Abbasi, Peter~L Bartlett, Varun Kanade, Yevgeny Seldin, and Csaba
  Szepesvari.
\newblock Online learning in {M}arkov decision processes with adversarially
  chosen transition probability distributions.
\newblock In \emph{Advances in Neural Information Processing Systems 26}, pages
  2508--2516, 2013.

\bibitem[Achiam and Sastry(2017)]{AchSa17}
Joshua Achiam and Shankar Sastry.
\newblock Surprise-based intrinsic motivation for deep reinforcement learning.
\newblock \emph{CoRR}, abs/1703.01732, 2017.

\bibitem[Azar et~al.(2019)Azar, Piot, Pires, Grill, Altch{\'{e}}, and
  Munos]{AzPiPi19}
Mohammad~Gheshlaghi Azar, Bilal Piot, Bernardo~A. Pires, Jean{-}Bastien Grill,
  Florent Altch{\'{e}}, and R{\'{e}}mi Munos.
\newblock World discovery models.
\newblock \emph{CoRR}, abs/1902.07685, 2019.

\bibitem[Baranes and Oudeyer(2009)]{Baranes2009}
A.~Baranes and P.-Y. Oudeyer.
\newblock {R-IAC}: Robust intrinsically motivated exploration and active
  learning.
\newblock \emph{IEEE Transactions on Autonomous Mental Development},
  1:\penalty0 155--169, 2009.

\bibitem[Burda et~al.(2019)Burda, Edwards, Pathak, Storkey, Darrell, and
  Efros]{BuEdPa19}
Yuri Burda, Harrison Edwards, Deepak Pathak, Amos~J. Storkey, Trevor Darrell,
  and Alexei~A. Efros.
\newblock Large-scale study of curiosity-driven learning.
\newblock In \emph{ICLR}, 2019.
\newblock URL \url{https://openreview.net/forum?id=rJNwDjAqYX}.

\bibitem[Even-dar et~al.(2005)Even-dar, Kakade, and Mansour]{NIPS2004_2730}
Eyal Even-dar, Sham~M Kakade, and Yishay Mansour.
\newblock Experts in a {M}arkov decision process.
\newblock In \emph{Advances in Neural Information Processing Systems}, pages
  401--408, 2005.

\bibitem[Gottlieb et~al.(2013)Gottlieb, Oudeyer, Lopes, and
  Baranes]{gottlieb2013information}
Jacqueline Gottlieb, Pierre-Yves Oudeyer, Manuel Lopes, and Adrien Baranes.
\newblock Information-seeking, curiosity, and attention: computational and
  neural mechanisms.
\newblock \emph{Trends in cognitive sciences}, 17\penalty0 (11):\penalty0
  585--593, 2013.

\bibitem[Haber et~al.(2018)Haber, Mrowca, Wang, Fei-Fei, and
  Yamins]{haber2018learning}
Nick Haber, Damian Mrowca, Stephanie Wang, Li~F Fei-Fei, and Daniel~L Yamins.
\newblock Learning to play with intrinsically-motivated, self-aware agents.
\newblock In \emph{Advances in Neural Information Processing Systems}, pages
  8388--8399, 2018.

\bibitem[Hazan et~al.(2019)Hazan, Kakade, Singh, and Van~Soest]{HaKaSiVS19}
Elad Hazan, Sham Kakade, Karan Singh, and Abby Van~Soest.
\newblock Provably efficient maximum entropy exploration.
\newblock In Kamalika Chaudhuri and Ruslan Salakhutdinov, editors, \emph{ICML},
  volume~97, pages 2681--2691, 2019.

\bibitem[Houthooft et~al.(2016)Houthooft, Chen, Duan, Schulman, De~Turck, and
  Abbeel]{houthooft2016variational}
Rein Houthooft, Xi~Chen, Yan Duan, John Schulman, Filip De~Turck, and Pieter
  Abbeel.
\newblock Variational information maximizing exploration.
\newblock In \emph{NIPS 2016 Deep Learning Symposium}, 2016.

\bibitem[Kober et~al.(2013)Kober, Bagnell, and
  Peters]{doi:10.1177/0278364913495721}
Jens Kober, J.~Andrew Bagnell, and Jan Peters.
\newblock Reinforcement learning in robotics: A survey.
\newblock \emph{The International Journal of Robotics Research}, 32\penalty0
  (11):\penalty0 1238--1274, 2013.

\bibitem[Lim and Auer(2012)]{ucbexp}
Shiau~Hong Lim and Peter Auer.
\newblock Autonomous exploration for navigating in {MDP}s.
\newblock In \emph{Proceedings of the 25th Annual Conference on Learning
  Theory}, volume~23 of \emph{Proceedings of Machine Learning Research}, pages
  40.1--40.24, 2012.

\bibitem[Lopes et~al.(2012)Lopes, Lang, Toussaint, and
  Oudeyer]{lopes2012exploration}
Manuel Lopes, Tobias Lang, Marc Toussaint, and Pierre-Yves Oudeyer.
\newblock Exploration in model-based reinforcement learning by empirically
  estimating learning progress.
\newblock In \emph{Advances in neural information processing systems}, pages
  206--214, 2012.

\bibitem[{Niroui} et~al.(2019){Niroui}, {Zhang}, {Kashino}, and
  {Nejat}]{8606991}
F.~{Niroui}, K.~{Zhang}, Z.~{Kashino}, and G.~{Nejat}.
\newblock Deep reinforcement learning robot for search and rescue applications:
  Exploration in unknown cluttered environments.
\newblock \emph{IEEE Robotics and Automation Letters}, 4\penalty0 (2):\penalty0
  610--617, 2019.

\bibitem[Ortner et~al.(2019)Ortner, Gajane, , and Auer]{UAI2019}
Ronald Ortner, Pratik Gajane, , and Peter Auer.
\newblock Variational regret bounds for reinforcement learning.
\newblock In \emph{Proceedings of the 35th Conference on Uncertainty in
  Artificial Intelligence}, 2019.

\bibitem[Ostrovski et~al.(2017)Ostrovski, Bellemare, van~den Oord, and
  Munos]{ostrovski2017count}
Georg Ostrovski, Marc~G Bellemare, A{\"a}ron van~den Oord, and R{\'e}mi Munos.
\newblock Count-based exploration with neural density models.
\newblock In \emph{Proceedings of the 34th International Conference on Machine
  Learning}, pages 2721--2730, 2017.

\bibitem[Oudeyer et~al.(2007)Oudeyer, Kaplan, and Hafner]{Oudeyer2007b}
P-Y. Oudeyer, F.~Kaplan, and V.V. Hafner.
\newblock Intrinsic motivation systems for autonomous mental development.
\newblock \emph{IEEE Transactions on Evolutionary Computation}, 11:\penalty0
  265--286, 2007.

\bibitem[Oudeyer and Kaplan(2007)]{Oudeyer2007}
Pierre-Yves Oudeyer and Frederic Kaplan.
\newblock What is intrinsic motivation? a typology of computational approaches.
\newblock \emph{Frontiers in neurorobotics}, 1, 2007.

\bibitem[Pathak et~al.(2017)Pathak, Agrawal, Efros, and
  Darrell]{pathak2017curiosity}
Deepak Pathak, Pulkit Agrawal, Alexei~A Efros, and Trevor Darrell.
\newblock Curiosity-driven exploration by self-supervised prediction.
\newblock In \emph{Proceedings of the IEEE Conference on Computer Vision and
  Pattern Recognition Workshops}, pages 16--17, 2017.

\bibitem[Schmidhuber(2010)]{Schmidhuber2010}
J.~Schmidhuber.
\newblock {Formal theory of creativity, fun, and intrinsic motivation
  (1990–2010)}.
\newblock \emph{Autonomous Mental Development, IEEE Transactions on},
  2:\penalty0 230--247, 2010.

\bibitem[Schmidhuber(1991)]{Schmidhuber1991}
J\"{u}rgen Schmidhuber.
\newblock A possibility for implementing curiosity and boredom in
  model-building neural controllers.
\newblock In \emph{Proceedings of the first international conference on
  simulation of adaptive behavior on From animals to animats}, pages 222--227.
  MIT Press, 1991.

\bibitem[Singh et~al.(2004)Singh, Barto, and Chentanez]{Singh2004}
Satinder~P. Singh, Andrew~G. Barto, and Nuttapong Chentanez.
\newblock Intrinsically motivated reinforcement learning.
\newblock In \emph{NIPS}, 2004.

\bibitem[Singh et~al.(2010)Singh, Lewis, Barto, and Sorg]{Singh2010}
Satinder~P. Singh, Richard~L. Lewis, Andrew~G. Barto, and Jonathan Sorg.
\newblock Intrinsically motivated reinforcement learning: An evolutionary
  perspective.
\newblock \emph{IEEE T. Autonomous Mental Development}, 2:\penalty0 70--82,
  2010.

\bibitem[Stadie et~al.(2015)Stadie, Levine, and Abbeel]{StLeAb15}
Bradly~C. Stadie, Sergey Levine, and Pieter Abbeel.
\newblock Incentivizing exploration in reinforcement learning with deep
  predictive models.
\newblock \emph{CoRR}, abs/1507.00814, 2015.

\end{thebibliography}

\newpage

\iftoggle{long-version}{
\appendix
\renewcommand{\thesection}{\Roman{section}}
\renewcommand\thesubsection{\Alph{subsection}}
\section{Proof of Lemma \ref{lem:StreamCount}}
\label{Proof:StreamCount}
\begin{proof}
The number of initiated streams is equal to the highest stream number initiated so far. Let that be $\hat{\stream}$.
%As described in the algorithm a stream $\stream$ is initiated in the quantum $(\stream-1)^2 + 1$. 
Since $\hat{\stream}$ is initiated on or before $q$, $(\hat{\stream} -  1)^2 + 1 \leq q$ (see \ref{step:stream}\ref{step:init} in Figure \ref{fig:alg}) which is equivalent to, 
\begin{equation}
    %(\hat{\stream} -  1)^2 + 1 &\leq q \nonumber\\
    %%%%%%%%%%%%%%%%%%%%
    %(\hat{\stream} -  1)^2 &< q \nonumber \\
    %%%%%%%%%%%%%%%%%%%%
    \hat{\stream} < \sqrt{q} + 1. \label{eq:StreamCount1}
\end{equation}
Since $\hat{\stream} + 1$ has not been initiated yet, $q \leq (\hat{\stream} + 1 -1)^2$ which translates to,
\begin{equation}
    \hat{\stream} \geq \sqrt{q} \label{eq:StreamCount2}.
\end{equation}
    Recall that both $\hat{\stream}$ and $q$ are integers $\geq$1. If $q$ is a perfect square, the only integer satisfying both Eq. $\eqref{eq:StreamCount1}$ and $\eqref{eq:StreamCount2}$ is $\sqrt{q} = \lceil \sqrt{q} \rceil$. If $q$ is not a perfect square, then Eq. $\eqref{eq:StreamCount2}$ reduces to $\hat{\stream} > \sqrt{q}$. And the only integer satisfying $\sqrt{q} < \stream < \sqrt{q} + 1$ is $\lceil \sqrt{q}\rceil$.
 \end{proof}   
 
\section{Proof of Lemma \ref{lem:StreamAlloc}}
\label{Proof:StreamAlloc} 
 \begin{proof}
    Claim 1 is a direct result of Lemma \ref{lem:StreamCount}. We prove claim 2 by induction on $b$. 
    Base case: $b =1$. At the end of $q = b^2 = 1$, only $1$ stream has been initiated and it has been active for $1$ quantum. \\ 
    Inductive case: Let's assume that the claim is true for $b = \hat{b}$ i.e at the end of quantum $q = \hat{b}^2$, exactly $\hat{b}$ streams have been initiated and each of them has been active for $\hat{b}$ quantums. At the next quantum i.e $\hat{b}^2 + 1$, stream $\hat{b} + 1$ will be initiated by the initiation rule and it will be active for the next $b$ quantums due to the allocation rule (ii). At this point, we are at the end of quantum $(\hat{b}+1) \cdot \hat{b}$ and
    all the $\hat{b}+1$ initiated streams have each been active for $\hat{b}$ quantums. Next, by virtue of the allocation rule (iii), each of the $(b+1)$ streams will be allocated $1$ quantum each till we are the end of quantum $ ((\hat{b}+1) \cdot \hat{b}) + (\hat{b}+1) = (\hat{b}+1)^2$. 
\end{proof}
}{} % toggle long-version

\end{document}